\documentclass[letterpaper, 10 pt, conference]{ieeeconf}  

\IEEEoverridecommandlockouts                              

\usepackage{amssymb}
\usepackage{amsmath}
\usepackage[english]{babel}
\usepackage{float}
\usepackage{graphicx}
\usepackage{hyperref}
\usepackage{mathtools}
\usepackage{graphicx}
\usepackage[noadjust]{cite}
\usepackage[font=footnotesize,labelfont=footnotesize]{subcaption}
\usepackage[font=footnotesize,labelfont=footnotesize]{caption}
\usepackage{eqparbox}
\usepackage{subcaption}
\usepackage{caption}
\usepackage{comment}
\usepackage{authblk}
\usepackage{multirow}
\usepackage{environ}
\usepackage{graphicx}

\usepackage{amsthm, algorithm, algorithmicx} 
\usepackage[noend]{algpseudocode}
\usepackage[english]{babel}

\let\labelindent\relax

\usepackage[inline]{enumitem}
\newtheorem{innercustomgeneric}{\customgenericname}
\providecommand{\customgenericname}{}
\newcommand{\newcustomtheorem}[2]{%
  \newenvironment{#1}[1]
  {%
   \renewcommand\customgenericname{#2}%
   \renewcommand\theinnercustomgeneric{##1}%
   \innercustomgeneric
  }
  {\endinnercustomgeneric}
}

\newcustomtheorem{customthm}{Theorem}
\newcustomtheorem{customlemma}{Lemma}
\newcustomtheorem{customprop}{Proposition}

\theoremstyle{plain}  
\newtheorem{theorem}{Theorem}
\newtheorem{lemma}[theorem]{Lemma}

\newtheorem{proposition}[theorem]{Proposition}

\algnewcommand\algorithmicinput{\textbf{Input:}}
\algnewcommand\algorithmicoutput{\textbf{Output:}}
\algnewcommand\Input{\item[\algorithmicinput]}%
\algnewcommand\Output{\item[\algorithmicoutput]}%

\DeclareMathOperator*{\argmin}{arg\,min}

\usepackage{xcolor}
\usepackage{soul}
\newcommand{\remove}[1]{}

\title{\LARGE \bf
Integrating Online Learning and Connectivity Maintenance for Communication-Aware Multi-Robot Coordination
}

\author{Yupeng Yang$^1$, Yiwei Lyu$^2$, Yanze Zhang$^1$, Ian Gao$^1$, and Wenhao Luo$^1$
\thanks{$^*$This work was supported in part by the U.S. National Science Foundation under Grant CMMI-2301749.}
\thanks{$^1$The authors are with the Department of Computer Science, University of North Carolina at Charlotte, Charlotte, NC 28223, USA. Email: {\tt \{yyang52, yzhang94, igao, wenhao.luo\}@uncc.edu}}
\thanks{$^{2}$The author is with the Department of Electrical and Computer Engineering, Carnegie Mellon University, Pittsburgh, PA 15213, USA. Email: {\tt yiweilyu@andrew.cmu.edu}}}
\begin{document}

\maketitle
\thispagestyle{empty}
\pagestyle{empty}

\begin{abstract}
This paper proposes a novel data-driven control strategy for maintaining connectivity in networked multi-robot systems. Existing approaches often rely on a pre-determined communication model specifying whether pairwise robots can communicate given their relative distance to guide the connectivity-aware control design, which may not capture real-world communication conditions. To relax that assumption, we present the concept of Data-driven Connectivity Barrier Certificates, which utilize Control Barrier Functions (CBF) and Gaussian Processes (GP) to characterize the admissible control space for pairwise robots based on communication performance observed online. This allows robots to maintain a satisfying level of pairwise communication quality (measured by the received signal strength) while in motion. 
Then we propose a Data-driven Connectivity Maintenance (DCM) algorithm that combines (1) online learning of the communication signal strength and (2) a bi-level optimization-based control framework for the robot team to enforce global connectivity of the realistic multi-robot communication graph and minimally deviate from their task-related motions. We provide theoretical proofs to justify the properties of our algorithm and demonstrate its effectiveness through simulations with up to 20 robots.
\end{abstract}

\section{Introduction}
In multi-robot systems, establishing a well-connected communication network is critical to ensure effective collaboration through smooth peer-to-peer information exchange and interactions among robots~\cite{parker2016multiple}. Examples of applications include environmental monitoring~\cite{luo2018adaptive}, disaster response, collaborative search and rescue operations~\cite{drew2021multi}, and agricultural activities~\cite{polic2021compliant}. In communication-constrained environments with varying inter-robot communication quality, it is essential to adaptively constrain robots' motion to maintain group cohesion. Namely, the multi-robot communication graph remains \textit{globally connected} as one component, where any pairwise robots can communicate through a set of communication edges on the graph. 

\begin{figure}[ht]
\centering
\begin{subfigure}[b]{0.4\textwidth}
  \includegraphics[width=0.9\linewidth,height=0.4\linewidth]{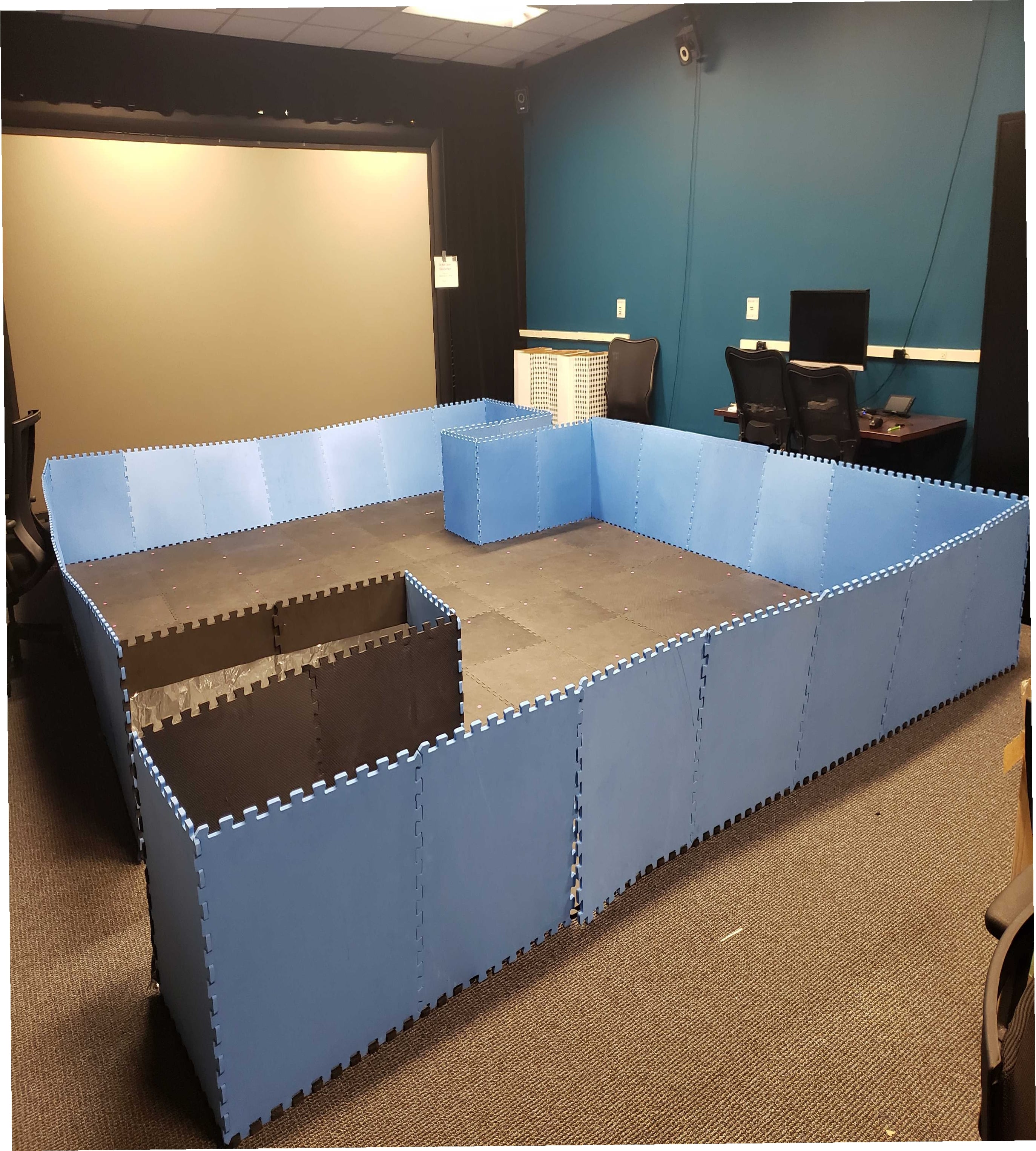}
 \caption{Real-World Environment Set-up}
  \label{fig1:real_world}
  \end{subfigure}
\begin{subfigure}[b]{0.2\textwidth}
  \includegraphics[width=1.0\linewidth]{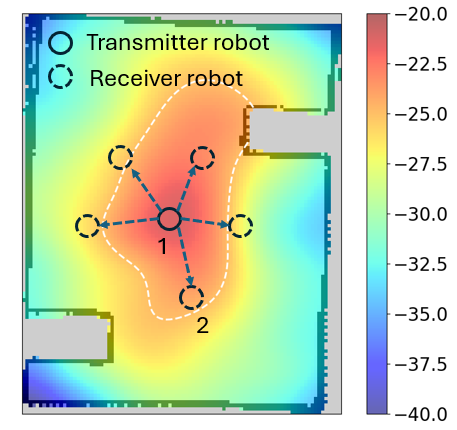}
 \caption{}
  \label{fig1:example_1}
   \end{subfigure}
\begin{subfigure}[b]{0.2\textwidth}
  \includegraphics[width=0.98\linewidth]{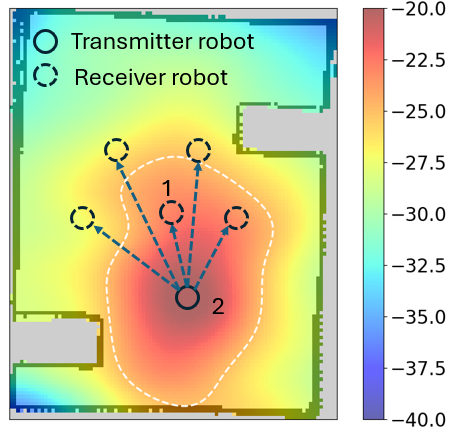}
 \caption{}
  \label{fig1:exampl_2}
   \end{subfigure}
  \caption{Example distributions of Received Signal Strength Index (RSSI) as inter-robot communication signal strength measurements in the real world. a) The lab environment depicting the real-world setting, b) RSSI value (dB) distribution over the space from the Wi-Fi transmitter Robot $1$, and c) RSSI value (dB) distribution over the space from the Wi-Fi transmitter Robot $2$. The white contour edges denote the level sets of RSSI = -25dB in (b) and (c), respectively, highlighting the asymmetric communication performance in real-world scenarios, e.g., the strength of wireless signal received by robot 2 from robot 1 is weaker than the one received by robot 1 from robot 2.}
  \label{fig:set_up}
\end{figure}

Previous research has predominantly concentrated on maintaining communication networks using disc-based or exponential function-based communication models, ensuring that pairs of robots remain within a predetermined distance to facilitate network connectivity~\cite{sabattini2013decentralized, luo2020behavior,wang2016multi,yang2023minimally,Yang-RSS-24}. However, such rigid assumptions may be unrealistic in the real world due to the sensitive wireless communication performance that can vary under different conditions. As depicted in Figures~\ref{fig1:example_1} and~\ref{fig1:exampl_2}, real-world scenarios often reveal that communication channels can be asymmetric and significantly influenced by environmental factors. This observation suggests that relying solely on simplified communication models may not guarantee reliable communication~\cite{fink2013robust}.

Data-driven approaches such as Gaussian Processes (GP) \cite{williams2006gaussian} have gained prominence for their effectiveness in modeling realistic and reliable communication channels \cite{banfi2017multirobot,clark2022propem}.
GP offers the advantage of providing a probabilistic framework, allowing for the incorporation of uncertainty into the communication model, thus leading to more robust and adaptable connectivity maintenance strategies. These GP-based approaches implicitly incorporate realistic factors such as path loss, shadowing, and fading effects, thereby providing a practical understanding of communication dynamics~\cite{fink2013robust, muralidharan2021communication}. For example, Gaussian Processes have been utilized in~\cite{fink2013robust} to accurately predict Received Signal Strength Indicator (RSSI) values, which is a direct indicator of the wireless signal's strength and quality.  
The approach in~\cite{fink2013robust} has allowed the formulation of connectivity maintenance strategies as second-order cone programs (SOCP). In addition, in~\cite{liu2017communication}, connectivity maintenance is approached by embedding a nonlinear constraint on the controller within a mixed-integer problem. Alternatively, a recent study \cite{clark2022propem} also utilizes physics-informed models, such as path loss models \cite{hata1980empirical}, to enhance the accuracy of RSSI predictions by incorporating signal propagation characteristics. Although these approaches have enabled communication-aware motion of a single robot \cite{muralidharan2021communication} or online communication modeling using a group of robots \cite{banfi2017multirobot}, how to design online learning and control strategies for efficient communication-aware multi-robot coordination remains a challenge.

To the best of our knowledge, few works enable robots to learn communication performance while progressing to their original task-related motion with provable guarantees in realistic communication connectivity. 
Due to the lack of information beforehand, robots must rely on communication quality data measured online to adaptively constrain their motion, so that any pairwise robots can effectively exchange information through the connected communication graph. 
On the other hand, it is desired that the connectivity maintenance 
minimally impacts the task-related robots' motion to provide the greatest flexibility for performing their original tasks.
Thus, our contributions are summarized as follows.
\begin{itemize}
\item A real-world point-to-point Received Signal Strength Indicator (RSSI) dataset is collected;
\item A novel Data-driven Connectivity Barrier Certificate is proposed, leveraging online RSSI data via GP and Control Barrier Functions~\cite{ames2019control, khan2022gaussian} to ensure "high-quality" communication between pairwise robots;
\item A novel Data-driven Connectivity Maintenance Algorithm (DCM) is presented to address a bi-level optimization problem, wherein the lower-level task focuses on selecting data-driven connectivity barrier certificate based control constraints, and the upper-level task resolves the resultant constrained control optimization problem to ensure global connectivity and collision avoidance among robot team with minimally disrupted robots' motion. Theoretical proofs and simulation results are presented to prove the efficacy of our approach.
\end{itemize}

\section{Preliminaries}\label{sec:pre}
In this paper, we consider an $N$-agent robotic team moving in a shared workspace, which consists of free space and occupied space $\mathcal{C}_{\mathrm{obs}}=\bigcup\limits_{k'=1}^{K'} \mathcal{O}_{k'}$ by $K'$ static polyhedral obstacles $\mathcal{O}_{k'} \subset \mathbb{R}^{d},\forall k'$, whose positions are assumed to be known by the robots. For each agent $i\in \mathcal{I}= \{1,..,N\}$, the dynamics can be captured by the following control affine form:
\begin{equation}
\begin{split}
\label{dynamics}
     &\dot{\mathbf{x}}_i = F_i(\mathbf{x}_i) +G_i(\mathbf{x}_i)\mathbf{u}_i, 
\end{split}
\end{equation}
where $\mathbf{x}_i\in\mathcal{X}_i\in\mathbb{R}^{d}$ denotes the state of agent $i$ and $\mathbf{u}_i\!\in\!\mathcal{U}_i\!\subset\!\mathbb{R}^{q}$ denotes the control input. $F_i:\mathbb{R}^{d} \mapsto \mathbb{R}^{d}$ and $G_i\!:\!\mathbb{R}^{d}\!\mapsto\!\mathbb{R}^{d\times q}$ are locally Lipschitz continuous.

\subsection{Collision Avoidance}
In this paper, it is assumed that $K'$ static polyhedral obstacles are represented as $L$ discretized obstacles, modeled as rigid spheres along the boundaries of the static obstacles \cite{thirugnanam2021duality}. Each discretized obstacle can be denoted as $o \in \{1,...,L\}$.
Then for the robot team, given the joint robot state $\mathbf{x} = \{\mathbf{x}_1,...,\mathbf{x}_N\} 
\in \mathcal{X} 
\subset \mathbb{R}^{dN}$, the discretized joint obstacle state $\mathbf{x}^\mathrm{obs}=\{\mathbf{x}^\mathrm{obs}_1,\ldots,\mathbf{x}^\mathrm{obs}_L\}\in\mathcal{X}^\mathrm{obs}\subset \mathbb{R}^{dL}$, the minimum inter-robot safe distance as $R_\mathrm{s}\in\mathbb{R}$, and the minimum obstacle-robot safe distance as $R_\mathrm{obs}\in\mathbb{R}$, the desired sets satisfying inter-robot or robot-obstacle collision avoidance can be defined as: 
\begin{align}
&h^\mathrm{s}_{i,j}(\mathbf{x}) = ||\mathbf{x}_i -\mathbf{x}_j||^2 -R_\mathrm{s}^2, \forall i>j, \notag\\
&\mathcal{H}^\mathrm{s}_{i,j} = \{\mathbf{x} \in \mathbb{R}^{dN} | h^\mathrm{s}_{i,j}(\mathbf{x}) \geq 0 \} \label{eq:h_safe_per}\\
&h^\mathrm{obs}_{i,o}(\mathbf{x},\mathbf{x}^\mathrm{obs}) = ||\mathbf{x}_i -\mathbf{x}^\mathrm{obs}_o||^2 -R_\mathrm{obs}^2, \forall i,o,\notag\\
&\mathcal{H}^\mathrm{obs}_{i,o} = \{\mathbf{x} \in \mathbb{R}^{dN}, \mathbf{x}^\mathrm{obs}\in\mathbb{R}^{dL}| h^\mathrm{obs}_{i,o}(\mathbf{x},\mathbf{x}^\mathrm{obs}_{o}) \geq 0 \} \label{eq:h_o_safe}
 \end{align}
 
To design control constraints that enforce forward invariance of the desired set $ \mathcal{H}(\mathbf{x})$, we summarize the results using Control Barrier Functions (CBF) \cite{ames2019control} as follows.
 
\begin{lemma}
\label{Lemma 1}\label{lem:cbf}[summarized from \cite{ames2019control}]
Given a dynamical system affine in control and a desired set $\mathcal{H}$ as the 0-superlevel set of a continuous differentiable function $h: \mathcal{X} \mapsto \mathbb{R}$, the function $h$ is called a control barrier function, if there exists an extended class-$\mathcal{K}$ function $\kappa(\cdot)$ such that $\sup_{\mathbf{u}\in\mathcal{U}}\{\dot{h}(\mathbf{x},\mathbf{u})+\kappa(h(\mathbf{x}))\}\geq 0$ for all $\mathbf{x} \in \mathcal{X}$. 
Any Lipschitz continuous controller $\mathbf{u}$ in the admissible control space $\mathcal{B}(\mathbf{x})$ rendering $\mathcal{H}$ forward invariant (i.e., keeping the system state $\mathbf{x}$ staying in $\mathcal{H}$ overtime) thus becomes:
$\mathcal{B}(\mathbf{x}) = \{ \mathbf{u}\in \mathcal{U} | \dot{h}(\mathbf{x},\mathbf{u}) + \kappa(h(\mathbf{x}))\geq 0 \}$.
\end{lemma} 
With Lemma~\ref{lem:cbf}, the admissible control space for robots to stay collision-free can thus be introduced as follows \cite{wang2017safety}.

\begin{align}
   & \mathcal{B}^\mathrm{s}(\mathbf{x}) = \{ \mathbf{u}\in \mathbb{R}^{qN} : \dot{h}^\mathrm{s}_{i,j}(\mathbf{x},\mathbf{u}) + \gamma h^\mathrm{s}_{i,j}(\mathbf{x})\geq 0, \forall i > j \} \label{eq:b_safe}\\
   & \mathcal{B}^\mathrm{obs}(\mathbf{x},\mathbf{x}^\mathrm{obs}) = \{ \mathbf{u}\in \mathbb{R}^{qN} : \notag\\
   &\qquad \qquad \qquad \dot{h}^\mathrm{obs}_{i,o}(\mathbf{x},\mathbf{x}^\mathrm{obs},\mathbf{u}) + \gamma h^\mathrm{obs}_{i,o}(\mathbf{x},\mathbf{x}^\mathrm{obs})\geq 0, \forall i,o \}\label{eq:b_obs}
\end{align}
where $\gamma$ is a user-defined parameter as the particular choice of $\kappa(h(\mathbf{x}))=\gamma h(\mathbf{x})$ as in \cite{luo2020behavior}. In \cite{wang2017safety} it is proven that the admissible controls of $\mathcal{B}^\mathrm{s}(\mathbf{x}),\mathcal{B}^\mathrm{obs}(\mathbf{x},\mathbf{x}^\mathrm{obs})$ make the safety sets $\mathcal{H}^\mathrm{s}_{i,j}$ and $\mathcal{H}^\mathrm{obs}_{i,o}$ forward invariant.

\subsection{Communication Model}\label{sec:communicaiton_model}
\begin{figure}[ht]
\centering
  \includegraphics[width=0.5\linewidth]{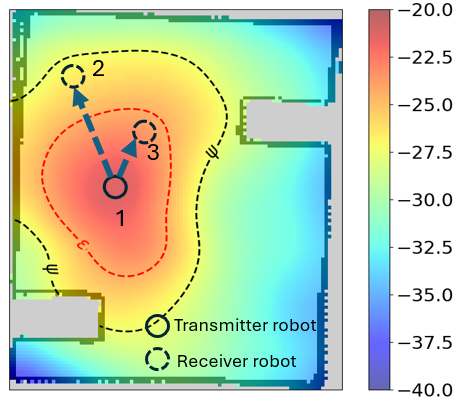}
 \caption{Example of our considered communication model: black and red contour edges indicate the level sets of RSSI values of $\psi$dB and $\epsilon$dB ($\psi<\epsilon$), respectively. $\mathcal{R}_{1,3}> \epsilon$ indicates that robot $3$ is strongly connected to transmitter robot $1$ with  $\mathcal{R}_{1,2}> \psi$ and $\mathcal{R}_{1,3}> \psi$ indicates that robots $3$ and $2$ can reliably measure the RSSI value of signals sent from robot $1$.}
  \label{fig:communication model}
\end{figure}
In this paper, the robot team communicates with each other through Wi-Fi, and we model the robot communication channel by estimating the received signal strength index (RSSI) \cite{fink2013robust}. 
We denote the RSSI value from transmitter robot $i$ to receiver robot $j$ as $\mathcal{R}_{i,j}(\mathbf{x})$. 
The closer the value $\mathcal{R}_{i,j}(\mathbf{x}) \in (-100,0) $ is to 0, the stronger the signal strength will be. 
We assume that if the $\mathcal{R}_{i,j}(\mathbf{x})$ is larger or equal than threshold $\psi \in \mathbb{R}^{-}$, then the robot $j$ can measure the RSSI value from robot $i$ using Wi-Fi antenna. To guarantee the receiver robot $j$ can be adequately connected to the transmitter robot $i$, we define the constraint $\mathcal{R}_{i,j}(\mathbf{x})\geq \epsilon$ to represent strong communication, where $\epsilon \in \mathbb{R}^{-}$ and $\epsilon > \psi$. In general, communication performance may not be symmetric (i.e., $\mathcal{R}_{i,j}(\mathbf{x})\neq\mathcal{R}_{j,i}(\mathbf{x})$). To guarantee robot $i$ and robot $j$ are strongly connected with each other, we require both $\mathcal{R}_{i,j}(\mathbf{x})$ and $\mathcal{R}_{j,i}(\mathbf{x})$ are larger than the threshold $\epsilon$. In this way, we have an undirected communication graph $\mathcal{G}=(\mathcal{V},\mathcal{E})$ where the undirected edge $(v_i,v_j) \in \mathcal{E}$ between robot $i$ and $j$ are \textit{strongly connected} (i.e., $(v_i,v_j)\Leftrightarrow(v_j,v_i)$ and $\mathcal{R}_{i,j}(\mathbf{x}),\mathcal{R}_{j,i}(\mathbf{x})\geq \epsilon$) with each node $v \in \mathcal{V}$ representing a robot. As depicted in Fig.~\ref{fig:communication model}, the robot $1$ is the transmitter robot and robot $2,3$ are receiver robots. Since $\mathcal{R}_{1,2}(\mathbf{x}) \geq \psi$ and $\mathcal{R}_{1,3}(\mathbf{x}) \geq \psi$, then robot $2$ and $3$ can both take the RSSI measurement from robot $1$. However, only robot $3$ is considered to be strongly connected to robot $1$ since only $\mathcal{R}_{1,3}(\mathbf{x}) \geq \epsilon$. For ease of notation, we only show the notion of dependence on time if necessary, i.e., $\mathcal{G} = \mathcal{G}(t)$. The graph is said to be \textbf{\textit{strongly connected}} or \textbf{\textit{globally connected}} if there is at least one path with strongly connected edges connecting any pairwise robots in the team.

\subsection{Gaussian Process Regression}
Gaussian Process (GP) regression \cite{williams2006gaussian} is a non-parametric approach that can learn complex functions without assuming a predetermined form for the underlying model. GP is characterized by a mean function, typically assumed to be zero for simplicity, and a covariance function $k(\mathbf{z},\mathbf{z}'):\mathcal{Z}\times\mathcal{Z}\mapsto\mathbb{R}$. Formally, a GP model can be defined as $\mathcal{GP}(0, k(\mathbf{z}, \mathbf{z'}))$,
where $\mathcal{Z} \subset \mathbb{R}^{a}$ denotes the input domain. The covariance matrix $k(\mathbf{z},\mathbf{z}')$ in $\mathcal{GP}$ is used to measure the similarity between input points $\mathbf{z}$ and  $\mathbf{z'}$ which requires to be positive semi-definite \cite{williams2006gaussian}.

Consider a dataset of \(Q\) observations with input vectors $X$
and corresponding scalar observations \(y \in \mathbb{R} \). We arrange this dataset as $\mathcal{D} =\{\mathbf{X}_{q'} , \mathbf{y}_{q'} \}_{q'=1}^Q$, where $X_Q= \{\mathbf{X}_{q'}\}^{Q}_{q'=1}$ and $y_Q =\{\mathbf{y}_{q'}\}^{Q}_{q'=1}$. Using GP, one can derive the posterior mean and variance for any deterministic query input \(\mathbf{X}_* \in \mathbb{R}^n\), based on previous observations. The posterior mean \(\mu\) and variance \(\sigma^2\) for the input $\mathbf{x}_*$ are provided by: 
\begin{align}
&\mu(\mathbf{x}_*) = \mathbf{k}(\mathbf{X}_*)^T K^{-1} y_Q \label{eq:h_mu} \\
&\sigma^2(\mathbf{x}_*) =  k(\mathbf{X}_*, \mathbf{X}_*) -  \mathbf{k}(\mathbf{X}_*)^T K^{-1}  \mathbf{k}(\mathbf{X}_*) \label{eq:h_sigma}
\end{align}
where \( \mathbf{k}(\mathbf{X}_*) = [k(\mathbf{X}_1, \mathbf{X}_*), \ldots, k(\mathbf{X}_{Q}, \mathbf{X}_*)]^T \in \mathbb{R}^Q\) is the covariance vector between $X_Q$ and \(\mathbf{X}_*\), and \(K \in \mathbb{R}^{Q \times Q}\), with entries \(k_{q',p} = k(\mathbf{X}_{q'}, \mathbf{X}_p)\), for \(q', p \in \{1, \ldots, Q\}\), denoting the covariance matrix for the input dataset \(X_Q\)\cite{williams2006gaussian}. 

\subsection{Gaussian Control Barrier Function}
Using the posterior mean \(\mu\) and variance \(\sigma^2\) from $\mathcal{GP}$, a function $h_\text{gp}(\mathbf{x})$ can be constructed below to define the safety region for the robot \cite{khan2022gaussian}.
\begin{equation}
   h_\mathrm{gp}(\mathbf{x}) = \mu(\mathbf{x}) - \sigma^{2}(\mathbf{x}) \label{eq:gaussian}
\end{equation}

Then the desired set for robot to stay in the 0-superlevel set of $h_\mathrm{gp}(\mathbf{x})$ can be defined as follows.
\begin{equation}
   \mathcal{H}_\mathrm{gp} = \{\mathbf{x} \in \mathbb{R}^{d}\; |\; h_\mathrm{gp}(\mathbf{x}) \geq 0 \} \label{eq:h_gaussian}
\end{equation}

\begin{lemma}
\label{lem:g_cbf}[summarized from \cite{khan2022gaussian}]
Given a dynamical system affine in control and a desired set $\mathcal{H}_\mathrm{gp}$ as the 0-superlevel set of a continuous differentiable function $h_\mathrm{gp}(\mathbf{x}) : \mathcal{X} \rightarrow \mathbb{R}$, the function $h_\mathrm{gp}(\mathbf{x})$ is called a Gaussian control barrier function, 
if it is constructed in the form of (\ref{eq:gaussian}) using the posterior mean and variance of Gaussian Process
with an infinitely mean-square differentiable positive definite kernel, $k(\mathbf{x}_i,\mathbf{x}_j): \mathcal{X}\times \mathcal{X} \mapsto \mathbb{R}$, and if there exists an extended class-$\mathcal{K}$ function $\kappa(\cdot)$ such that $\sup_{\mathbf{u}\in\mathcal{U}}\{\dot{h}_\mathrm{gp}(\mathbf{x},\mathbf{u})+\kappa(h_\mathrm{gp}(\mathbf{x})))\}\geq 0$ for all $\mathbf{x} \in \mathcal{X}$. 
Any Lipschitz continuous controller $\mathbf{u}$ in the admissible control space $\mathcal{S}(\mathbf{x})$ rendering $\mathcal{H}_\mathrm{gp}$ forward invariant thus becomes: $\mathcal{S}(\mathbf{x}) = \{\mathbf{u}\in \mathcal{U} | \dot{h}_\mathrm{gp}(\mathbf{x},\mathbf{u}) + \kappa(h_\mathrm{gp}(\mathbf{x}))\geq 0 \}$.
\end{lemma} 

With Lemma~\ref{lem:g_cbf}, the admissible control space for robots to stay in $\mathcal{H}_\mathrm{gp}$ can thus be introduced as follows \cite{khan2022gaussian}.
\begin{align}
   \mathcal{S}(\mathbf{x}) = \{ \mathbf{u}\in \mathbb{R}^{qN} : \dot{h}_\mathrm{gp}(\mathbf{x},\mathbf{u}) + \gamma h_\mathrm{gp}(\mathbf{x})\geq 0 \} \label{eq:gp_safe}
\end{align}

\subsection{Problem Statement}
In this paper, robots are assumed to be governed by a nominal task-related controller $\mathbf{u}^\text{ref} = \{\mathbf{u}^\text{ref}_{1},...,\mathbf{u}^\text{ref}_{N}\}\subset\mathbb{R}^{qN}$
to achieve their goal positions. Each robot is equipped with a Wi-Fi transmitter and receiver to gather RSSI data $\mathcal{R}_{i,j}(\mathbf{x})$ ($i$ and $j$ denote the transmitter robot $i$ and receiver robot $j$, respectively). To ensure that robots perform their assigned tasks while maintaining connectivity in environments with unknown pairwise robot signal strength indicated by RSSI value, we aim to 1) model a continuously differentiable function $h^\text{gp}_{i,j}(\mathbf{x})$ to approximate $h^\text{*gp}_{i,j} = \mathcal{R}_{i,j}(\mathbf{x}) -\epsilon$ whose 0-superlevel set delineate regions within the environment where the receiver robot $j$ is strongly connected to the transmitter robot $i$; 
2) determine the $\mathcal{G}^\mathrm{c*}$ which is the particular graph among all subgraphs of $\mathcal{G}$ and is minimally disrupting the nominal robots' controllers if being maintained, and 3) establish control constraints that enforce the selected $\mathcal{G}^\mathrm{c*}$ is strongly connected over time which results in the $\mathcal{G}\supseteq \mathcal{G}^\mathrm{c*}$ always strongly connected. Therefore, the problem can be summarized as follows.

\noindent\textit{\textbf{Problem Statement:}} 
Given a mobile robot team operating in a shared environment with nominal task-related motion and using the real-time RSSI data $\mathcal{R}_{i,j}(\mathbf{x})$ from any Wi-Fi transmitter robot $i$ and receiver robot $j$ ($\forall i,j$), we seek to learn a function $h^\text{gp}_{i,j}(\mathbf{x})$ to design corresponding control constraints and obtain the control $\mathbf{u}^*$ that minimally deviates from $\mathbf{u}^\text{ref}$ yet ensures 
$\mathcal{G}$ is strongly connected over time.

\section{Method}
\subsection{Data-driven Connectivity Barrier Certificates}
Contrary to the existing work, e.g., \cite{wang2016multi, luo2020behavior} where robots are considered connected if they fall within a predetermined distance, we introduce a data-driven method focusing on building communication between robots with RSSI using an online learning-based method. Specifically, the robot utilizes real-time RSSI measurements along with the state information of the robots to dynamically learn and encode a function $h^\text{gp}_{i,j}(\mathbf{x})$ and $h^\text{gp}_{j,i}(\mathbf{x})$ for pairwise robot communication. 
This approach enables the development of control constraints that guarantee the connectivity of the multi-robot system without relying on specific distance thresholds. In this paper, we use GP to model the continuously differentiable function $h^\text{gp}_{i,j}(\mathbf{x})$ to approximate $h^\text{*gp}_{i,j}(\mathbf{x})= \mathcal{R}_{i,j}(\mathbf{x})- \epsilon$ that $h^\text{gp}_{i,j}(\mathbf{x})\geq 0$ indicates robot $j$ is strongly connected to robot $i$ under the pairwise setting. 
As a data-driven method, it is desired to consider the inherent uncertainty in estimating of function $h^\text{gp}_{i,j}(\mathbf{x})$. Hence, we desire to formulate the function $h^\text{gp}_{i,j}(\mathbf{x})$ using the posterior mean \(\mu\) and variance \(\sigma^2\) of $\mathcal{GP}$ as $h^\text{gp}_{i,j}(\mathbf{x}) = \mu_{i,j}(\mathbf{x}) - \sigma_{i,j}^2(\mathbf{x})$. 

To learn such a function, the next step is to construct a reasonable dataset. 
Assuming that in a team of $N$ robots, there are $Q$ pairs of robots that the receiver robot $j$ is capable of measuring the RSSI information from the transmitter robot $i$. We formally define the dataset as $\mathcal{D} =\{\mathbf{X}_{q'} , \mathbf{y}_{q'} \}_{q'=1}^Q$, where $X_Q= \{\mathbf{X}_{q'}\}^{Q}_{q'=1}$ and $y_Q =\{\mathbf{y}_{q'}\}^{Q}_{q'=1}$. The ${\mathbf{X}}_{q'}$ is the transmitter-receiver robot position pair, which is denoted as $\mathbf{X}_{q'} = (\mathbf{x}^{q'}_i, \mathbf{x}^{q'}_j)$, where $q' \in [1,Q]$, $\mathbf{x}^{q'}_i$ is the position of transmitter robot, $\mathbf{x}^{q'}_j$ is the position of receiver robot, and $\mathbf{y}_{q'} = \mathcal{R}_{i,j}(\mathbf{X}_{q'}) - \epsilon$.
Note that we assume that the robot team is strongly connected initially, hence the dataset $\mathcal{D}$ is guaranteed to be non-empty. We select the squared exponential (SE) function $k_\mathrm{se}$ as our kernel function, which is an infinitely mean-square differentiable positive definite kernel~\cite{khan2022gaussian} and has been widely utilized in predicting RSSI values, e.g., \cite{hahnel2006gaussian}:
\begin{equation}
k_\mathrm{se}(\mathbf{X}_{q'}, \mathbf{X}) = \sigma_f^2 \exp\left(-\frac{d(\mathbf{X}_{q'}, \mathbf{X})^2}{2l^2}\right)
\end{equation}
where $\mathbf{X} = (\mathbf{x}_i,\mathbf{x}_j)$, and $\sigma_f$ and $l$ are hyper-parameters for $\mathcal{GP}$. The distance function $d(\mathbf{X}_{q'}, \mathbf{X})$ denotes the distance between $\mathbf{X}_{q'}$ and $\mathbf{X}$, which is defined as \cite{liu2017communication}:
\begin{equation}
    d(\mathbf{X}_{q'}, \mathbf{X}) = \left\|\begin{pmatrix} \mathbf{x}^{q'}_{i} \\ \mathbf{x}^{q'}_j \end{pmatrix} - \begin{pmatrix} \mathbf{x}_i \\ \mathbf{x}_j \end{pmatrix}\right\|
\end{equation}

With this, our $h^\mathrm{gp}_{i,j}(\mathbf{x})$ can be defined as $h^\mathrm{gp}_{i,j}(\mathbf{x}) = \mathbf{k}_\mathrm{se}(\mathbf{X})^T K_\mathrm{se}^{-1} \mathbf{y}_Q -( k_\mathrm{se}(\mathbf{X}, \mathbf{X}) -  \mathbf{k}_\mathrm{se}(\mathbf{X})^T K_\mathrm{se}^{-1}  \mathbf{k}_\mathrm{se}(\mathbf{X}))$, where \( \mathbf{k}_\mathrm{se}(\mathbf{X}) = [k_\mathrm{se}(\mathbf{X}_1, \mathbf{X}), \ldots, k_\mathrm{se}(\mathbf{X}_{Q}, \mathbf{X})]^T \in \mathbb{R}^Q\) and  \(K_\mathrm{se} \in \mathbb{R}^{Q \times Q}\) with entries \(k^{q',p}_\mathrm{se} = k_\mathrm{se}(\mathbf{X}_{q'}, \mathbf{X}_p)\), for \(q', p \in \{1, \ldots, Q\}\). Then the desired set on $\mathbf{x}$ for any receiver robot $j$ is strongly connected to transmitter robot $i$ can be defined as:
\begin{align}
\mathcal{H}^\mathrm{c}_{i,j} &= \{\mathbf{x} \in \mathbb{R}^{dN}: h^\text{gp}_{i,j}(\mathbf{x}) \geq 0\} \label{eq:H_conn_per}
\end{align}

Recall that we assume that edge $(v_i,v_j)$ between robots \( i \) and \( j \) are strongly connected if and only if both $\mathcal{R}_{i,j}(\mathbf{x})$ and $\mathcal{R}_{j,i}(\mathbf{x})$ are larger than threshold $\epsilon$. Then for the entire team with any given communication spanning graph $\mathcal{G}^{\mathrm{c}}=(\mathcal{V},\mathcal{E}^\mathrm{c})\subseteq \mathcal{G}$ to enforce, the desired set implying the robot team is strongly connected can thus be defined as: 
\begin{align}
    \mathcal{H}^\mathrm{c}(\mathcal{G}^\mathrm{c}) =\bigcap_{(v_i,v_j)\in\mathcal{E}^\mathrm{c}} \!(\mathcal{H}^\mathrm{c}_{i,j}\cap\mathcal{H}^\mathrm{c}_{j,i})\label{eq:graph_inter}
\end{align}

Following Lemma~\ref{lem:cbf}, we formally define our Data-driven Connectivity Barrier Certificate as follows.
\begin{lemma}
\label{lem:date_Driven_connectivity}
\textbf{Data-driven Connectivity Barrier Certificate:} Given any communication spanning graph $\mathcal{G}^{\mathrm{c}}=(\mathcal{V},\mathcal{E}^\mathrm{c})\subseteq \mathcal{G}$, and a desired set $\mathcal{H}^\mathrm{c}(\mathcal{G}^\mathrm{c}) $ in~(\ref{eq:graph_inter}) with $h^\text{gp}_{i,j}(\mathbf{x})$, for any Lipschitz continuous controller $\mathbf{u}$, the Data-driven Connectivity Barrier Certificates as admissible control space $ \mathcal{S}_\mathbf{u}(\mathbf{x},\mathcal{G}^\mathrm{c})$ defined below renders $\mathcal{H}^\mathrm{c}(\mathcal{G}^\mathrm{c})$ forward invariant:

{\footnotesize\begin{align}\label{eq:data_cbc}
\mathcal{S}_\mathbf{u}(\mathbf{x},\mathcal{G}^\mathrm{c}) = &\{\mathbf{u}\in\mathbb{R}^{dN}:\dot{h}^\mathrm{gp}_{i,j}(\mathbf{x},\mathbf{u})+\gamma h^\mathrm{gp}_{i,j}(\mathbf{x})\geq 0, \forall(v_i,v_j)\in\mathcal{E}^\mathrm{c}\} \cap \notag\\
& \{\mathbf{u}\in\mathbb{R}^{dN}:\dot{h}^\mathrm{gp}_{j,i}(\mathbf{x},\mathbf{u})+\gamma h^\mathrm{gp}_{j,i}(\mathbf{x})\geq 0 ,\forall(v_j,v_i)\in\mathcal{E}^\mathrm{c}\} 
\end{align}}
\end{lemma}

\begin{proof}
We first demonstrate that $h^\text{gp}_{i,j}(\mathbf{x})$ (and $h^\text{gp}_{j,i}(\mathbf{x})$ likewise) is a valid Gaussian control barrier function following~\cite{khan2022gaussian}.
Since the robots are initially strongly connected, we can guarantee that the dataset is always not empty. Hence, our selected kernel enjoys the property of infinite mean-square differentiability. With this and according to the sample path differentiability theorem \cite{eriksson2018scaling}, the partial derivative of~$h^\text{gp}_{i,j}(\mathbf{x})$ with respect to $\mathbf{x}$ at any transmitter-receiver pairwise query point $\mathbf{X} = (\mathbf{x}_i,\mathbf{x}_j)$ can be calculated as:
{\footnotesize\begin{align}
&\left.\frac{\partial h_{i,j}^{\text{gp}}(\mathbf{x})}{\partial \mathbf{x}}\right|_{\mathbf{X}} = \mathbf{y}_Q^\top K_\mathrm{se}^{-1} \left.\frac{\partial  \mathbf{k}_\mathrm{se}(\mathbf{X})}{\partial \mathbf{x}}\right|_{\mathbf{X}} + 2  \mathbf{k}_\mathrm{se}(\mathbf{X})^\top K_\mathrm{se}^{-1} \left.\frac{\partial  \mathbf{k}_\mathrm{se}(\mathbf{X})}{\partial \mathbf{x}}\right|_{\mathbf{X}} \label{eq:gp_derivate}
\end{align}}
where the kernel derivative in~(\ref{eq:gp_derivate}) at any query point $\mathbf{X}$ is given by:
\begin{equation}
\left.\frac{\partial \mathbf{k}_\mathrm{se}(\mathbf{X})}{\partial \mathbf{x}}\right|_{\mathbf{X}} =
\begin{bmatrix}
-\frac{1}{l^2}  (\mathbf{X}_1-\mathbf{X})^{\top}\mathbf{k}_\mathrm{se}(\mathbf{X}_1, \mathbf{X}) \\
\vdots \\
-\frac{1}{l^2}  (\mathbf{X}_q-\mathbf{X})^{\top}\mathbf{k}_\mathrm{se}(\mathbf{X}_q, \mathbf{X}) 
\end{bmatrix}
\end{equation}
Considering the robot dynamics, $\dot{h}^\text{gp}_{i,j}(\mathbf{x},\mathbf{u})$ finally becomes: 
\begin{equation}
    \dot{h}^\text{gp}_{i,j}(\mathbf{x},\mathbf{u}) = \frac{\partial h_{i,j}^{\text{gp}}(\mathbf{x})}{\partial \mathbf{x}}\begin{bmatrix}
        F_i(\mathbf{x}_i)\\F_j(\mathbf{x}_j)
    \end{bmatrix} + \frac{\partial h_{i,j}^{\text{gp}}(\mathbf{x})}{\partial \mathbf{x}}\begin{bmatrix}G_i(\mathbf{x}_i)\mathbf{u}_i\\   G_j(\mathbf{x}_j)\mathbf{u}_j
    \end{bmatrix}
    \label{eq:gaussian_partial_derivate}
\end{equation}

Each single linear constraint $\dot{h}^\text{gp}_{i,j}(\mathbf{x},\mathbf{u})+\gamma h^\text{gp}_{i,j}(\mathbf{x})\geq 0$ on $\mathbf{u}$ actually divides the control space into two half spaces, so considering the unbounded control input, it is always feasible to find the $\mathbf{u}$ that satisfy the control constraints $\dot{h}^\text{gp}_{i,j}(\mathbf{x},\mathbf{u})+\gamma h^\text{gp}_{i,j}(\mathbf{x})\geq 0$. In this way, we can verify that $h^\text{gp}_{i,j}(\mathbf{x})$ is a valid control barrier function. Then, consider the property of the CBF in Lemma~\ref{lem:cbf}, to enforce the intersection set $\mathcal{H}^\mathrm{c}(\mathcal{G}^\mathrm{c})$ in~(\ref{eq:graph_inter}) forward invariant so that the required strong connectivity can be satisfied, we need to guarantee that the control space constrained by $\mathcal{S}_\mathbf{u}(\mathbf{x},\mathcal{G}^\mathrm{c})$ in~(\ref{eq:data_cbc}) is always not empty. Our data-driven connectivity barrier certificates rely on the composition of the set of valid control barrier functions $h^\text{gp}_{i,j}(\mathbf{x})$. It has been proved in~\cite{egerstedt2018robot} that given any unbounded control, any controller leading the $\dot{\mathbf{x}} = \mathbf{0}$ could always be the non-empty interior of their corresponding control space. Considering any bounded control input, authors in~\cite{xiao2022sufficient} provide a sufficient condition to guarantee the feasibility of CBF-based QPs.
Thus, we can conclude that if $\mathbf{u}\in\mathcal{S}_\mathbf{u}(\mathbf{x},\mathcal{G}^\mathrm{c})$, then it can render $\mathcal{H}^\mathrm{c}(\mathcal{G}^\mathrm{c})$ in~(\ref{eq:graph_inter}) forward invariant. 
\end{proof}

\subsection{Bi-level Optimization}
In this section, we formulate the problem as a bi-level optimization problem to 1) dynamically find the lower level optimal subgraph $\mathcal{G}^\mathrm{c*}$ from real-time strongly connected graph $\mathcal{G}$ to maintain, and 2) minimally modify the upper level task-related controller to enforce the robot system to stay safe and strongly connected. 

At each time step, there may exist multiple subgraphs $\mathcal{G}^\mathrm{c}$ which are strongly connected in the strongly connected graph $\mathcal{G}$, hence one desires to maintain one particular graph $\mathcal{G}^\mathrm{c*}$ which introduces the least control constraint. To maintain each edge $(v_i,v_j)\in\mathcal{G}^\mathrm{c}$ in a candidate graph $\mathcal{G}^\mathrm{c}$, in our problem settings, it will introduce two control constraints in~(\ref{eq:data_cbc}). 
Therefore, the desired graph $\mathcal{G}^\mathrm{c*}$, characterized by edges that represent the least likely to be violated connectivity constraints under the nominal controllers $\mathbf{u}^\mathrm{ref}$, can be found within the collection of all spanning trees $\mathcal{T}$ of $\mathcal{G}$, who has the minimum number of edges ($N-1$) to be preserved. To quantitatively and heuristically assess the level of violating the pairwise connectivity constraints under the nominal controller $\mathbf{u}^\mathrm{ref}$, we hereby introduce the weight $w_{i,j}$:
\begin{align}
    w_{i,j} = -(\dot{h}^\mathrm{gp}_{i,j}(\mathbf{x},\mathbf{u}^\mathrm{ref}) +\gamma h^\mathrm{gp}_{i,j}(\mathbf{x}) + \dot{h}^\mathrm{gp}_{j,i}(\mathbf{x},\mathbf{u}^\mathrm{ref}) +\gamma h^\mathrm{gp}_{j,i}(\mathbf{x})) \label{eq:weight_gaussian}
\end{align}

With the smaller value of weight $w_{i,j}$, it heuristically indicates the less violated control constraints. Given this edge weight definition, the weighted graph $\mathcal{G}$ is thus denoted as $\mathcal{G}=(\mathcal{V},\mathcal{E},\mathcal{W})$ with $\mathcal{W}=\{w_{i,j}\}$. We now redefine each spanning tree $\mathcal{T}^\mathrm{c}\in\mathcal{T}$ as a weighted spanning tee $\mathcal{T}^\mathrm{c}_w = (\mathcal{V},\mathcal{E}^{T},\mathcal{W}^{T})$. Then, the graph $\mathcal{G}^\mathrm{c*}$ can thus be defined as:
\begin{align}
\mathcal{G}^\mathrm{c*}\leftarrow \bar{\mathcal{T}}_w^\mathrm{c}=  \underset{\mathcal{T}^\mathrm{c}_{w} \in \mathcal{T}}{\arg\min} \sum_{(v_i,v_j) \in \mathcal{E}^T} w_{i,j}\label{eq:graph_optimal}
\end{align}

The optimal solution to~(\ref{eq:graph_optimal}) corresponds to the Minimum Spanning Tree (MST) within the communication graph $\mathcal{G}$, where the weights are determined by $\{w_{i,j}\}$. Finally, given the control bound $\alpha_i$ of robot $i$, we reformulate the step-wise bi-level optimization-based control problem as follows. 

\begin{align}
 &\mathbf{u}^* = \argmin_{\mathcal{G}^\mathrm{c},\mathbf{u}} \sum_{i=1}^{N}||\mathbf{u}_i-\mathbf{u}^\mathrm{ref}_i||^2 \label{eq:obj_final}\\
 \text{s.t.} \quad 
 &\mathcal{G}^\mathrm{c}\leftarrow{\mathcal{\bar{T}}}_w^\mathrm{c}= \argmin_{\{\mathcal{T}_w^\mathrm{c'}\}} \sum_{(v_i,v_j)\in \mathcal{E}^{T}}\{w_{i,j}\}\\
&\mathbf{u}\in\mathcal{B}^\mathrm{s}(\mathbf{x})  \textstyle \bigcap 
\mathcal{B}^\mathrm{obs}(\mathbf{x},\mathbf{x}^\mathrm{obs})\bigcap\mathcal{S}_\mathbf{u}(\mathbf{x},\mathcal{G}^\mathrm{c}), \notag\\
&\qquad\;||\mathbf{u}_i|| \leq \alpha_i,\forall i=1,\ldots,N  \notag
\end{align}

Building upon, we summarize our Data-driven Connectivity Maintenance (DCM) algorithm in Algorithm~\ref{alg:Dynamic}. Our algorithm is structured in a centralized manner, hence all robots collect RSSI information together to construct the dataset $\mathcal{D}$, which can be useful to train $h^\mathrm{gp}_{i,j}(\mathbf{x})$. Besides, we assume the $\mathcal{R}_{i,j}(\mathbf{x})$ is time-independent, so the previous time step data will still be stored into the dataset. At each time-step, the robots first collect all the RSSI information on Algorithm~\ref{alg:Dynamic} Line~\ref{alg:collected_Data} based on the data collection criteria as mentioned in Section~\ref{sec:communicaiton_model} and then determine the strongly connected graph $\mathcal{G}$. Following this, we adopt the standard MST algorithm to find the $\bar{\mathcal{T}}_w^\mathrm{c}$ on Algorithm~\ref{alg:Dynamic} Line~\ref{alg:line:tree}. Finally, the optimal controller $\mathbf{u}^{*}$ can be obtained by solving the resultant Quadratic Programming (QP) problem on Algorithm~\ref{alg:Dynamic} Line~\ref{alg:line:find_u_star}. It is worth to note that our control constraints $\mathcal{B}^\mathrm{s}(\mathbf{x})\bigcap\mathcal{B}^\mathrm{obs}(\mathbf{x},\mathbf{x}^\mathrm{obs})\bigcap \mathcal{S}_\mathbf{u}(\mathbf{x},\bar{\mathcal{T}}_w^\mathrm{c})$ manifest as \textbf{\textit{linear inequalities}} with respect to the control $\mathbf{u}$.

\begin{algorithm}
\caption{Data-driven Connectivity Maintenance Algorithm}
    \label{alg:Dynamic}
    \begin{algorithmic}[1]
    \Input{$\mathbf{x}$-the current states (positions) of the robots, $\mathbf{u}^\mathrm{ref}$-the nominal task-related multi-robot controller, $\mathcal{C}_\mathrm{obs}$ the occupied space of the obstacles,
    $\mathcal{D}$ $\gets$ $\mathcal{D}_0$ the initialized training dataset.
    }
    \Output{The desired minimally modified controller $\mathbf{u}^*\in\mathbb{R}^{dN}$ from (\ref{eq:obj_final})}
    \Function{DCM}{$\mathbf{x}$, $\mathbf{u}^\mathrm{ref}$, $\mathcal{C}_\mathrm{obs}$, $\mathcal{D}$}
    \For{Each Time Step}
    \For{$i = 1$ to $N$}
    \For{$j = 1$ to $N$}
    \If{$\mathcal{R}_{i,j}(\mathbf{x}) \geq \psi$ and $i\neq j$} \label{alg:collected_Data} 
    \State Add the {\{($\mathbf{x}_i$, $\mathbf{x}_j$),\;$\mathcal{R}_{i,j}(\mathbf{x})-\epsilon$\}) to the dataset $\mathcal{D}$}
    \EndIf
    \EndFor
    \EndFor
    \For { All Edges $(v_i,v_j)\in \mathcal{E}$ of current strongly connected communication graph $\mathcal{G}=(\mathcal{V},\mathcal{E})$}
     \State Weight assignment: $\mathcal{W}$ $\gets$ $w_{i,j}$ using~(\ref{eq:weight_gaussian})
     \EndFor
     \State Get weighted graph $\mathcal{G}=(\mathcal{V},\mathcal{E},\mathcal{W})$
     \State  Solve $\bar{\mathcal{T}}_w^\mathrm{c}=\argmin_{\{\mathcal{T}_w^\mathrm{c}\}} \sum_{(v_i,v_j)\in \mathcal{E}^{T}} w_{i,j}$ by standard MST algorithm: $\bar{\mathcal{T}}_w^\mathrm{c}\gets$ MST($\mathcal{G}$)\label{alg:line:tree} 
     \State Construct $\mathcal{S}_\mathbf{u}(\mathbf{x},\bar{\mathcal{T}}_w^\mathrm{c})$ using $h^\mathrm{gp}_{i,j}(\mathbf{x})$ and $\dot{h}^\mathrm{gp}_{i,j}(\mathbf{x},\mathbf{u})$ and (\ref{eq:gaussian_partial_derivate}) 
    \State  \Return{$\mathbf{u}^* =\argmin_\mathbf{u} \sum_{i=1}^{N}||\mathbf{u}_i-\mathbf{u}^\mathrm{ref}_i||^2 $} where $\mathbf{u}\in \mathcal{B}^\mathrm{s}(\mathbf{x})\bigcap\mathcal{B}^\mathrm{obs} (\mathbf{x},\mathbf{x}^\mathrm{obs})\bigcap \mathcal{S}_\mathbf{u}(\mathbf{x},\bar{\mathcal{T}}_w^\mathrm{c}),||\mathbf{u}_i|| \leq \alpha, \forall i=1,\ldots,N $ \label{alg:line:find_u_star}
     \EndFor
    \EndFunction
    \end{algorithmic}
\end{algorithm}
\setlength{\textfloatsep}{0pt}

\begin{figure*}[ht]
\centering  
\begin{subfigure}[b]{0.3\textwidth}
\centering
  \includegraphics[width=0.8\linewidth, height=0.8\linewidth]{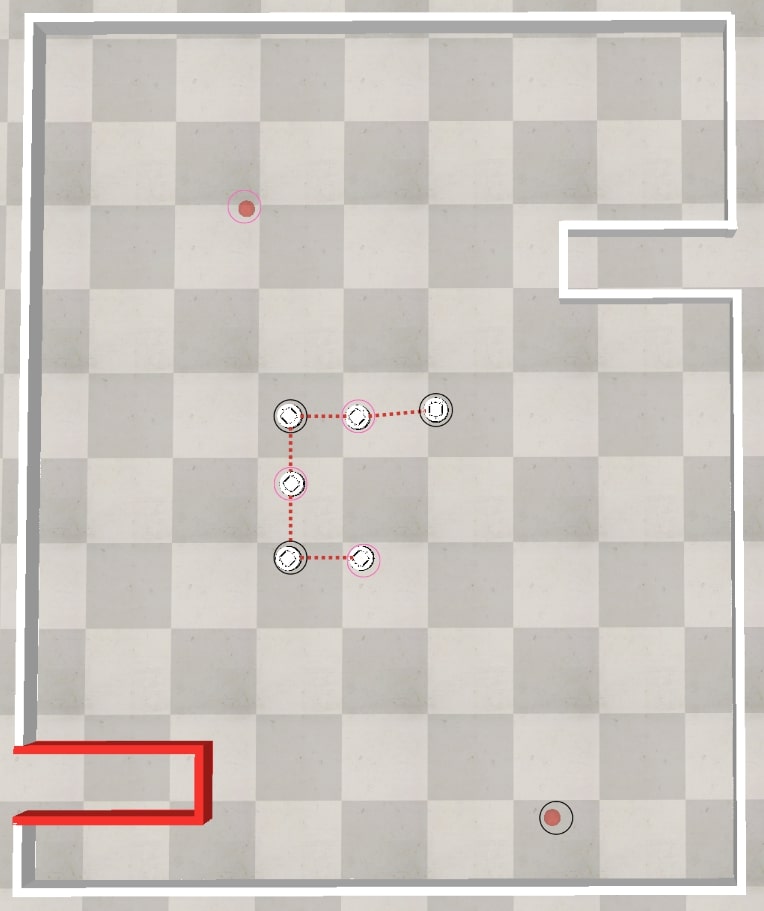}
 \caption{Our Method $t = 0$}
  \label{fig2:subfiga}
  \end{subfigure}
\begin{subfigure}[b]{0.3\textwidth}
\centering
  \includegraphics[width=0.8\linewidth, height=0.8\linewidth]{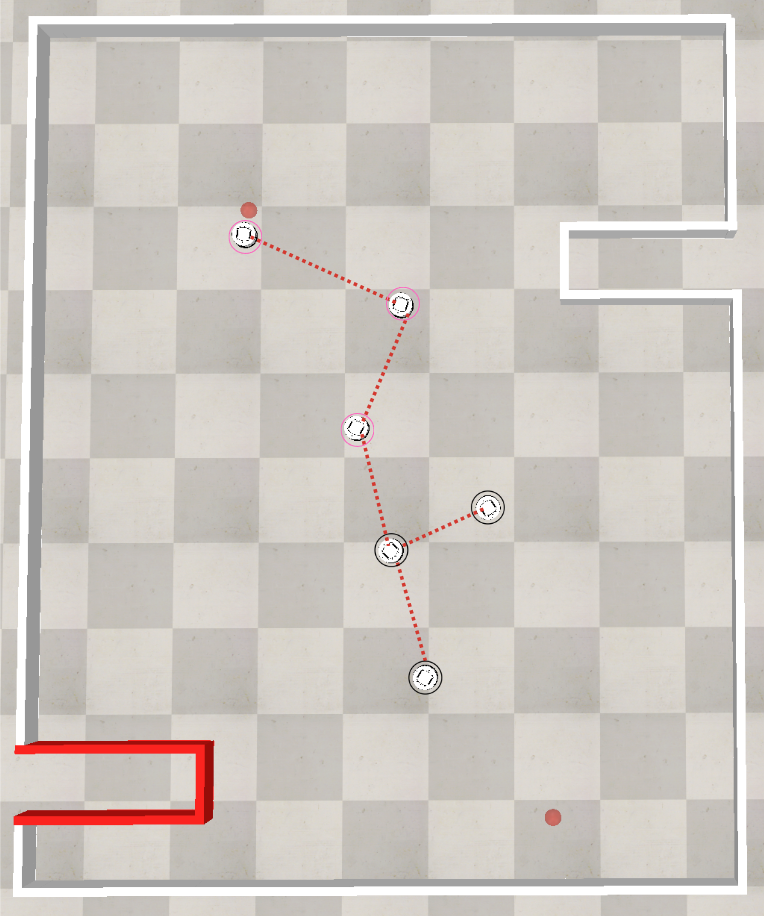}
 \caption{Our Method $t = 350$}
  \label{fig2:subfigb}
   \end{subfigure}
   \begin{subfigure}[b]{0.3\textwidth}
   \centering
\includegraphics[width=0.8\linewidth, height=0.8\linewidth]{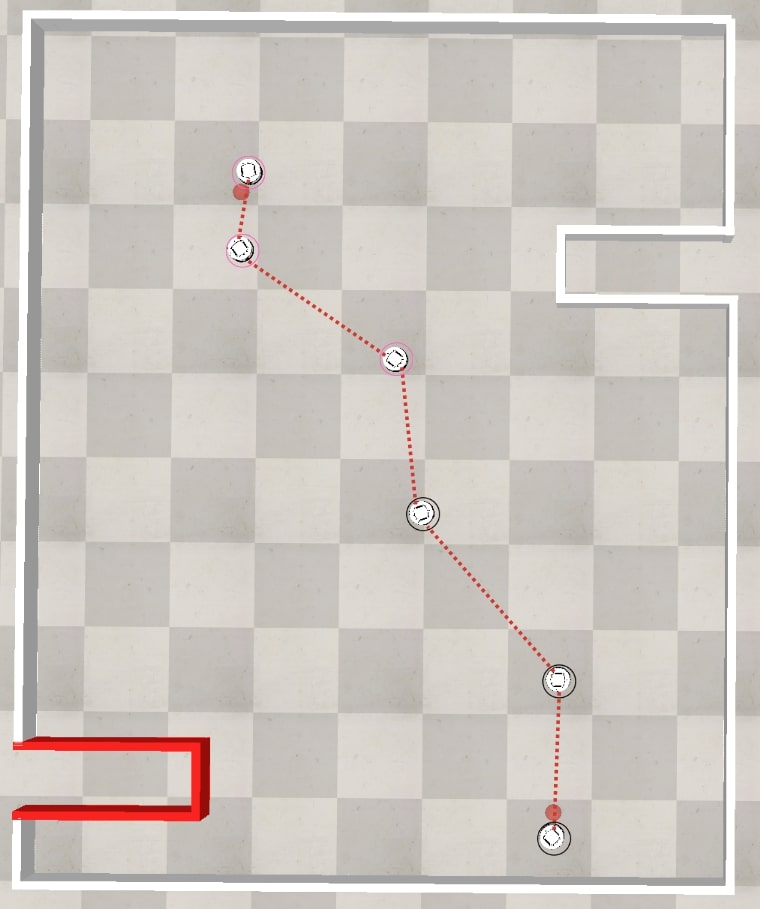}
 \caption{Our Method $t=700$ (Converged)}
  \label{fig2:subfigc}
  \end{subfigure}\\
\begin{subfigure}{0.3\textwidth}
\centering
\includegraphics[width=0.8\linewidth, height=0.8\linewidth]{img/cop_init.jpg}
\caption{MCCST}
  \label{fig3:subfigure1}
  \end{subfigure}
    \begin{subfigure}{0.3\textwidth}
    \centering
\includegraphics[width=0.8\linewidth, height=0.8\linewidth]{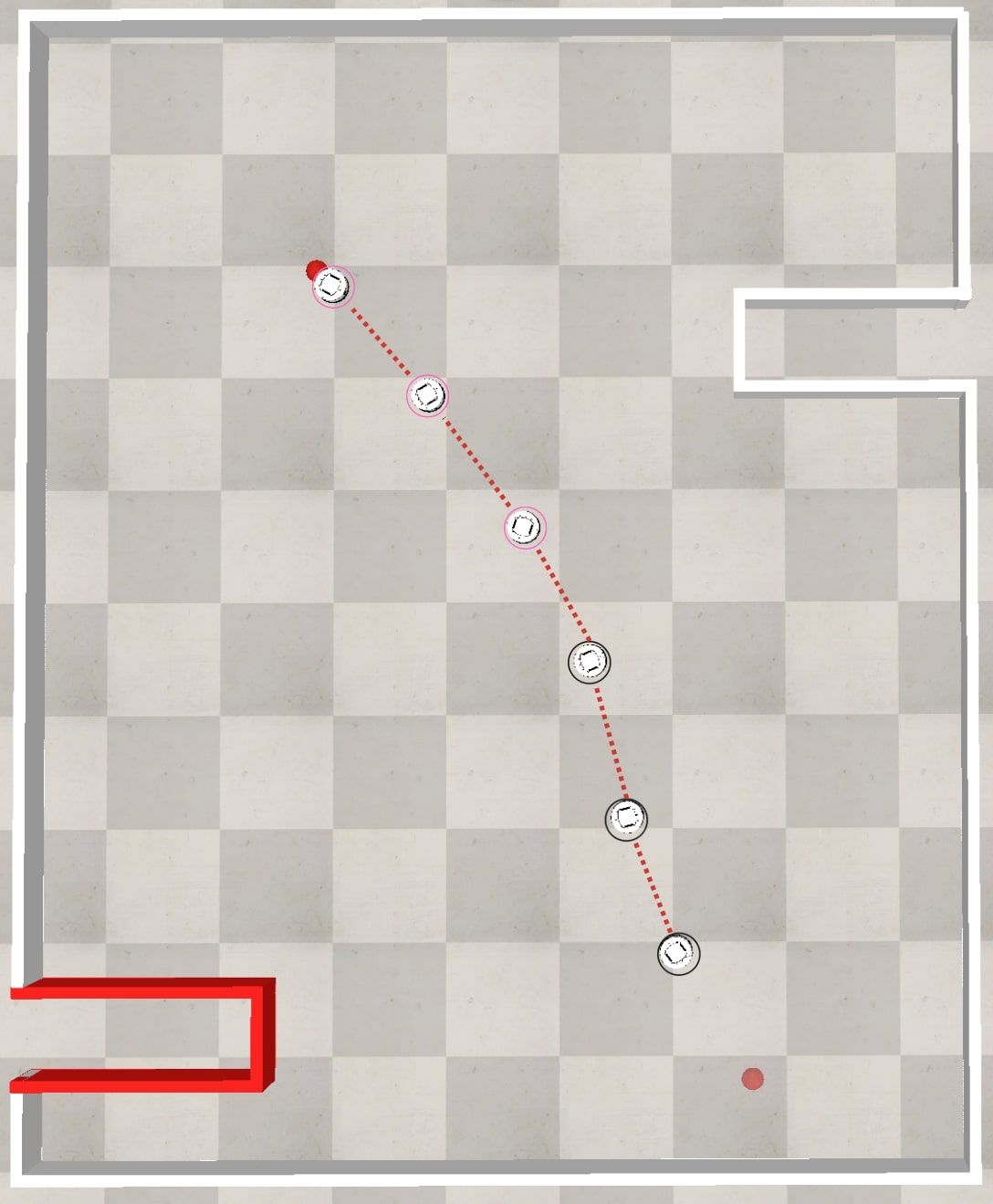}
 \caption{MCCST $t= 379$ ($R_\mathrm{c} = 0.7$m)}
  \label{fig3:subfigure2}
  \end{subfigure}
  \begin{subfigure}{0.3\textwidth}
  \centering
\includegraphics[width=0.8\linewidth, height=0.8\linewidth]{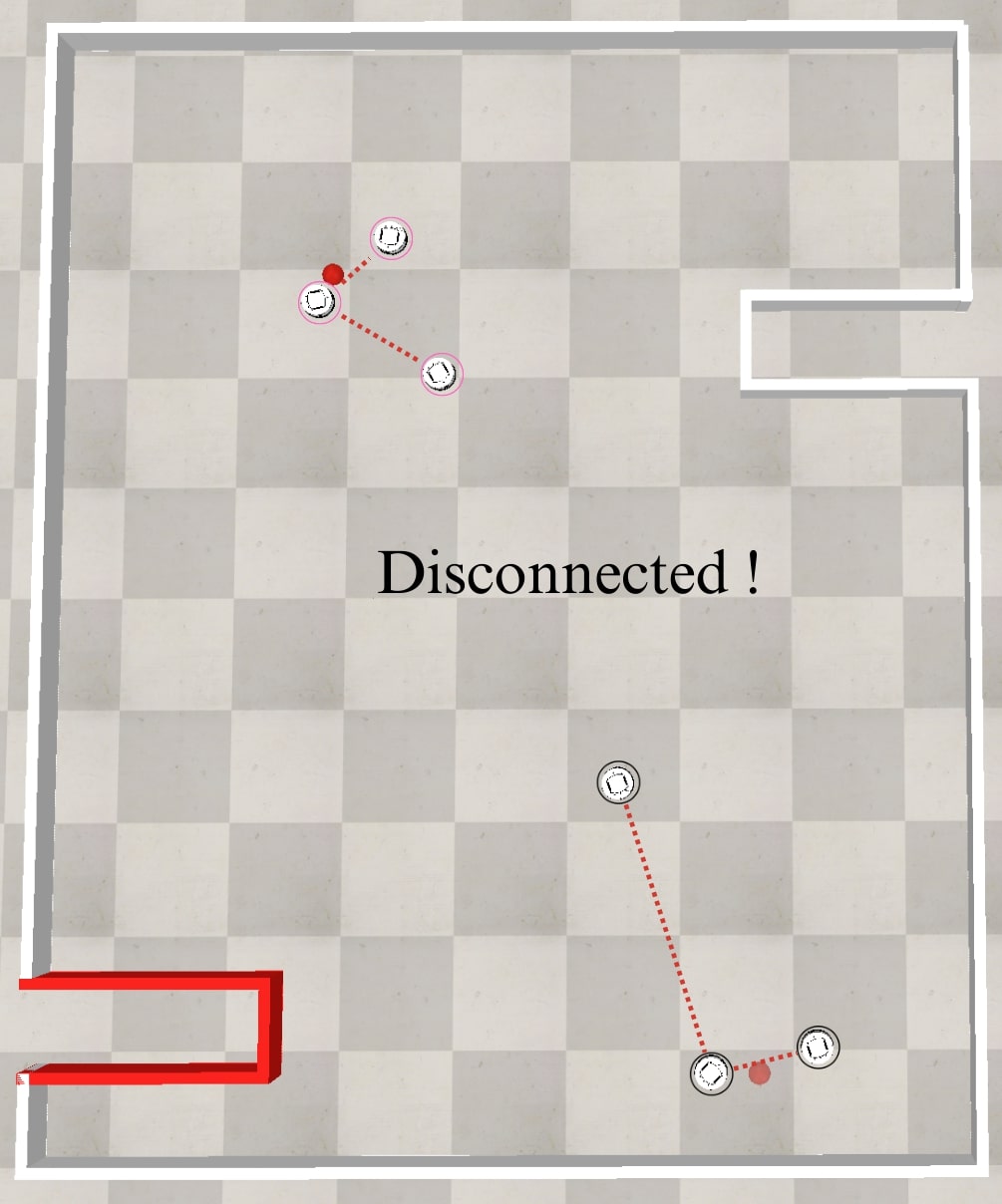}
\caption{MCCST  $t= 700$ ($R_\mathrm{c} = 1.2$m)}
  \label{fig3:subfigure3}
  \end{subfigure}
  \caption{Simulation example of 5 robots and tasked to corresponding colored task places. The red dash lines in this figure denote the currently optimal strongly connected communication graph $\mathcal{G}^\mathrm{c*}$. The white and red boxes represent the obstacles. The robot diameter is $0.16$m. Compared baseline method MCCST \cite{luo2020behavior} under different parameter settings include (e) MCCST ($R_\mathrm{c} = 0.7$m), and (f) MCCST ($R_\mathrm{c} = 1.2$m).}.
  \label{fig:our_simulation}
\end{figure*}
\begin{figure*}[ht]
\centering
\begin{subfigure}[b]{0.3\textwidth}
\centering
  \includegraphics[width=0.8\linewidth]{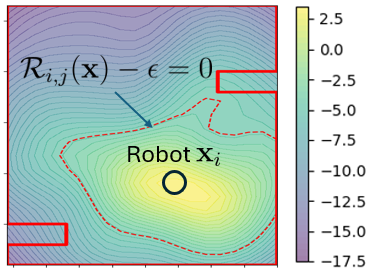}
  \centering
 \caption{Ground truth data}
  \label{fig:ground}
  \end{subfigure}
\begin{subfigure}[b]{0.3\textwidth}
\centering
  \includegraphics[width=0.8\linewidth]{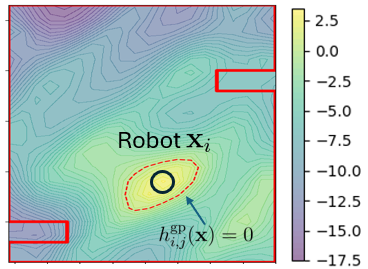}
 \caption{$h^\mathrm{gp}_{i,j}(\mathbf{x})$ at $t =0$}
  \label{fig:prediction_init}
\end{subfigure}
\begin{subfigure}[b]{0.3\textwidth}
\centering
  \includegraphics[width=0.8\linewidth]{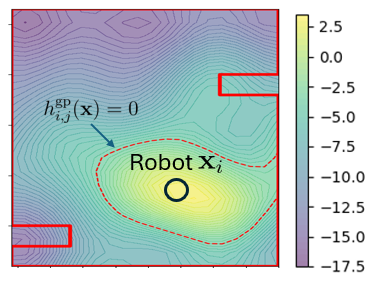}
 \caption{$h^\mathrm{gp}_{i,j}(\mathbf{x})$ at $t =700$}
  \label{fig:prediciton_final}
\end{subfigure}
  \caption{The contour plot for the $h^\mathrm{gp}_{i,j}(\mathbf{x})$ at different time step: a) the ground truth 0-superlevel set calculated by $\mathcal{R}_{i,j}(\mathbf{x})-\epsilon$, b) the predicted $h^\mathrm{gp}_{i,j}(\mathbf{x})$ with 0-superlevel set at $t =0$ denoted as red dash line, and c) the predicted $h^\mathrm{gp}_{i,j}(\mathbf{x})$ with 0-superlevel set at $t=700$ denoted as red dash line.} 
  \label{fig:prediction}
\end{figure*}

\begin{figure*}[ht]
\centering
\begin{subfigure}[b]{0.28\textwidth}
\centering
  \includegraphics[trim={0.55cm 0.5cm 0.5cm 0.5cm}, clip,width=1\linewidth]{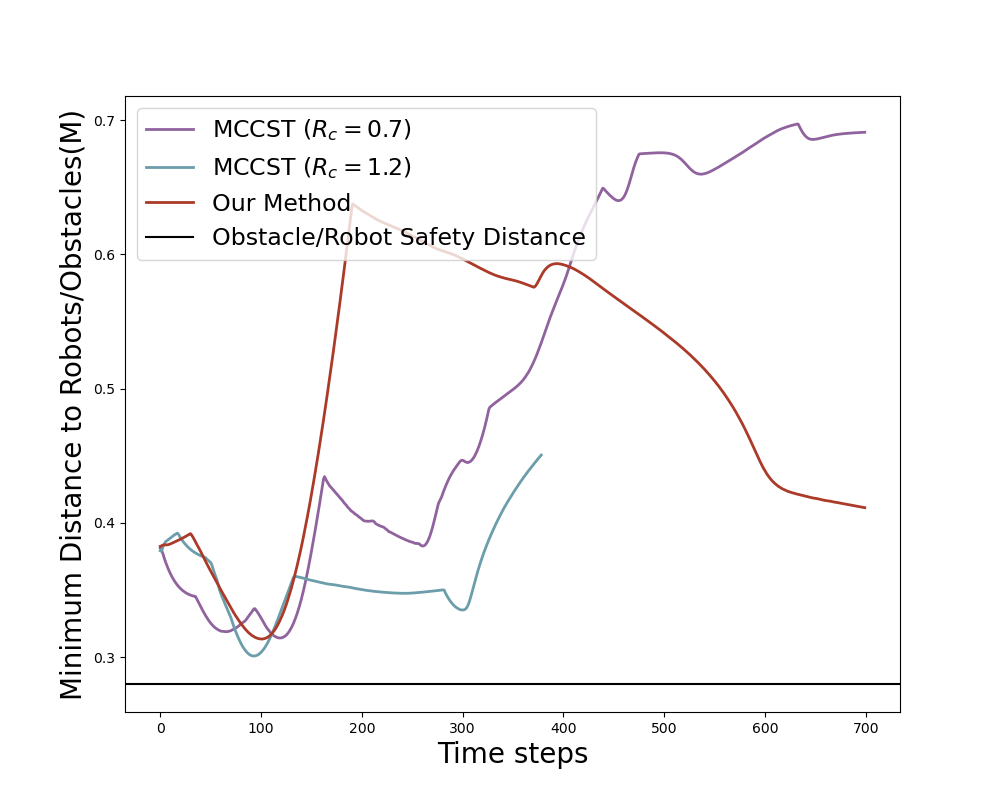}
 \caption{$D_{min}$ to robots/obstacles}
  \label{fig3:safety}
  \end{subfigure}
\begin{subfigure}[b]{0.28\textwidth}
\centering
  \includegraphics[trim={0.55cm 0.5cm 0.5cm 0.5cm}, clip,width=1\linewidth]{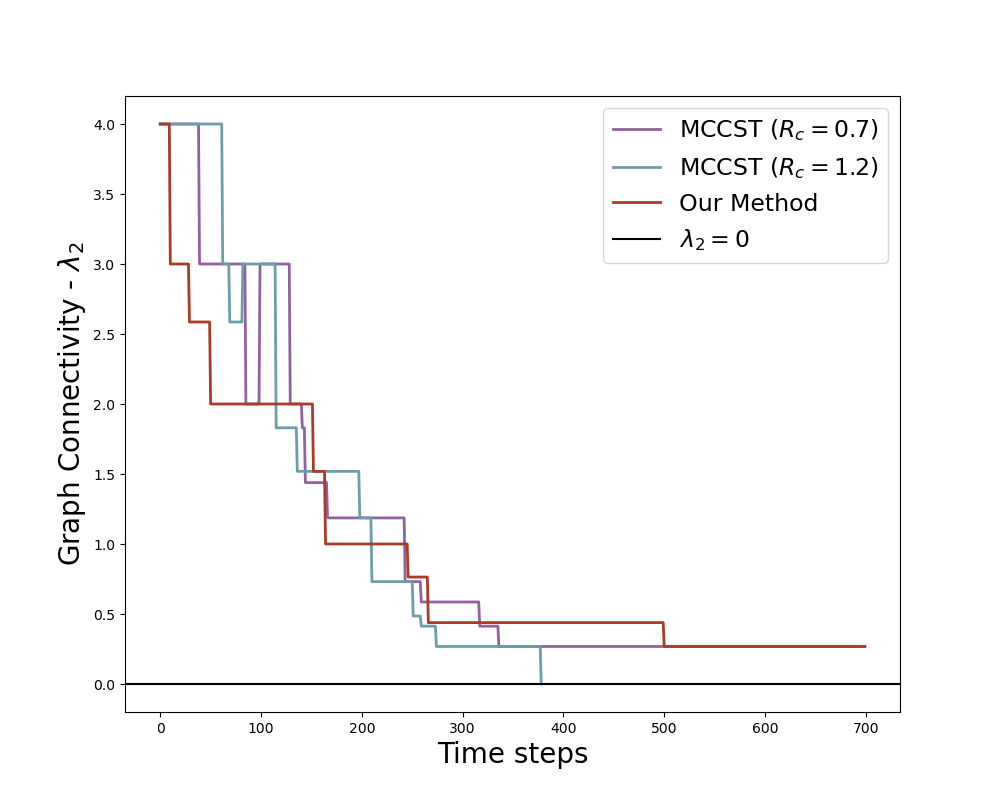}
 \caption{Algebraic connectivity}
  \label{fig3:connectivity}
\end{subfigure}
\begin{subfigure}[b]{0.28\textwidth}
 \centering
  \includegraphics[trim={0.55cm 0.5cm 0.5cm 0.5cm}, clip,width=1\linewidth]{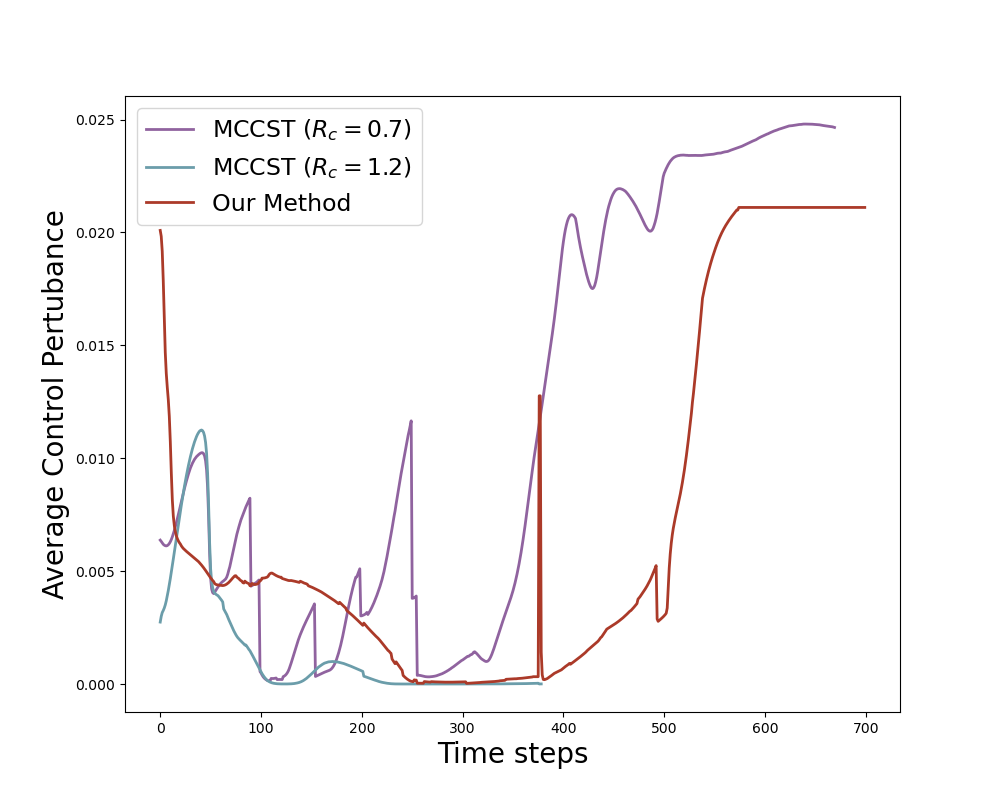}
 \caption{Average control perturbation}
  \label{fig3:control}
\end{subfigure}
  \caption{Comparative analysis of the simulation example shown in Fig.~\ref{fig:our_simulation} with respect to different metrics a) Minimum distance to robots/obstacles $D_{min}$ to verify the safety constraints' satisfaction ($R_\mathrm{s}$ = $R_\mathrm{obs} = 0.28$m), b) Average algebraic connectivity to indicate whether the graph is strongly connected ($\lambda_2 >0$) or not $\lambda_2 =0$, where $\lambda_2$ is the second-smallest eigenvalue of the Laplacian matrix calculated from the adjacency matrix. The elements in our adjacency matrix indicate whether the pairwise robots $i$ and $j$ are strongly connected with each other ($\mathcal{R}_{i,j}(\mathbf{x})\geq -25$dB and $\mathcal{R}_{j,i}(\mathbf{x})\geq -25$dB), and c) Average control perturbation (computed by $\frac{1}{N}\sum_{i=1}^N|| \mathbf{u}_i-\mathbf{u}^\mathrm{ref}_i||^{2}$ measuring the accumulated deviation from nominal controllers).}
  \label{fig:numerical}
\end{figure*}

\begin{figure*}[ht]
\centering
\begin{subfigure}[b]{0.28\textwidth}
 \centering
  \includegraphics[trim={0.55cm 0.5cm 0.5cm 0.5cm}, clip,width=1\linewidth]{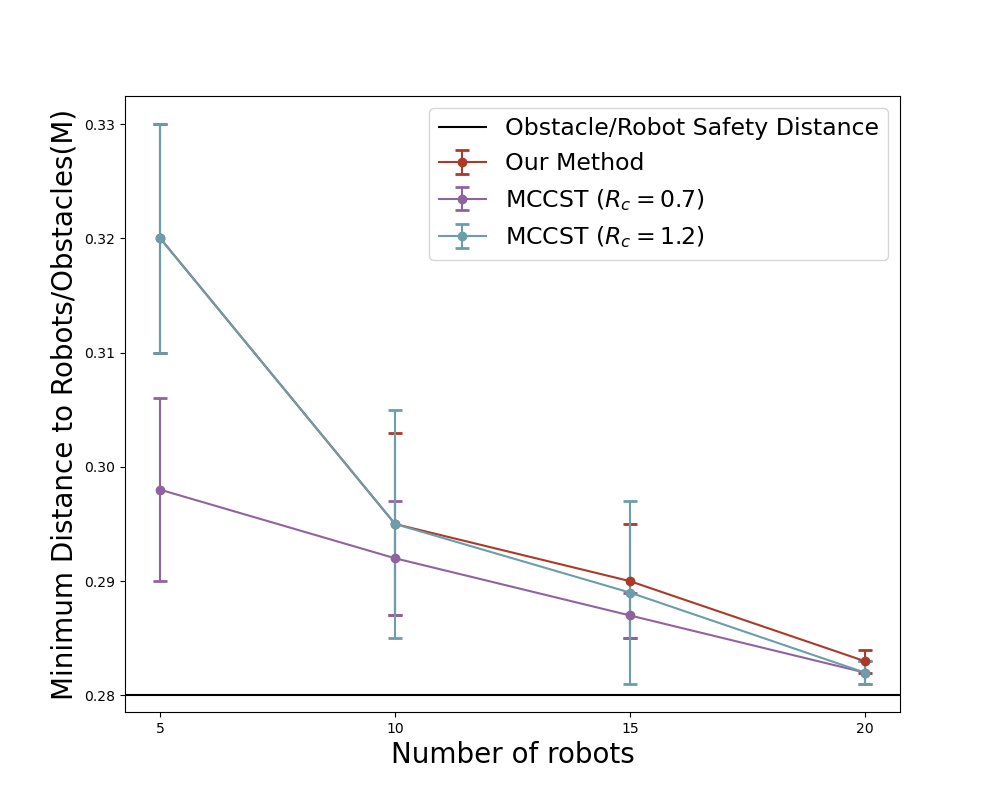}
 \caption{$D_{min}$ to robots/obstacles}
  \label{fig:safe_error}
  \end{subfigure}
\begin{subfigure}[b]{0.28\textwidth}
 \centering
  \includegraphics[trim={0.55cm 0.5cm 0.5cm 0.5cm}, clip,width=1\linewidth]{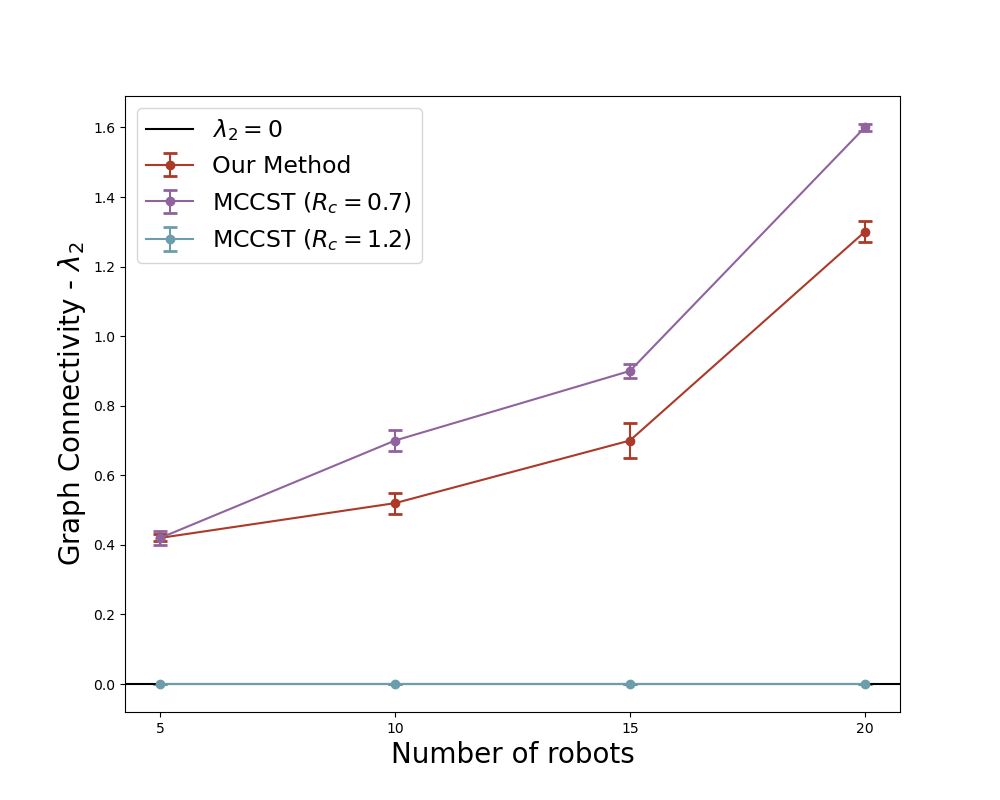}
 \caption{Algebraic connectivity}
  \label{fig:connectivity_error}
\end{subfigure}
\begin{subfigure}[b]{0.28\textwidth}
 \centering
  \includegraphics[trim={0.55cm 0.5cm 0.5cm 0.5cm}, clip,width=1\linewidth]{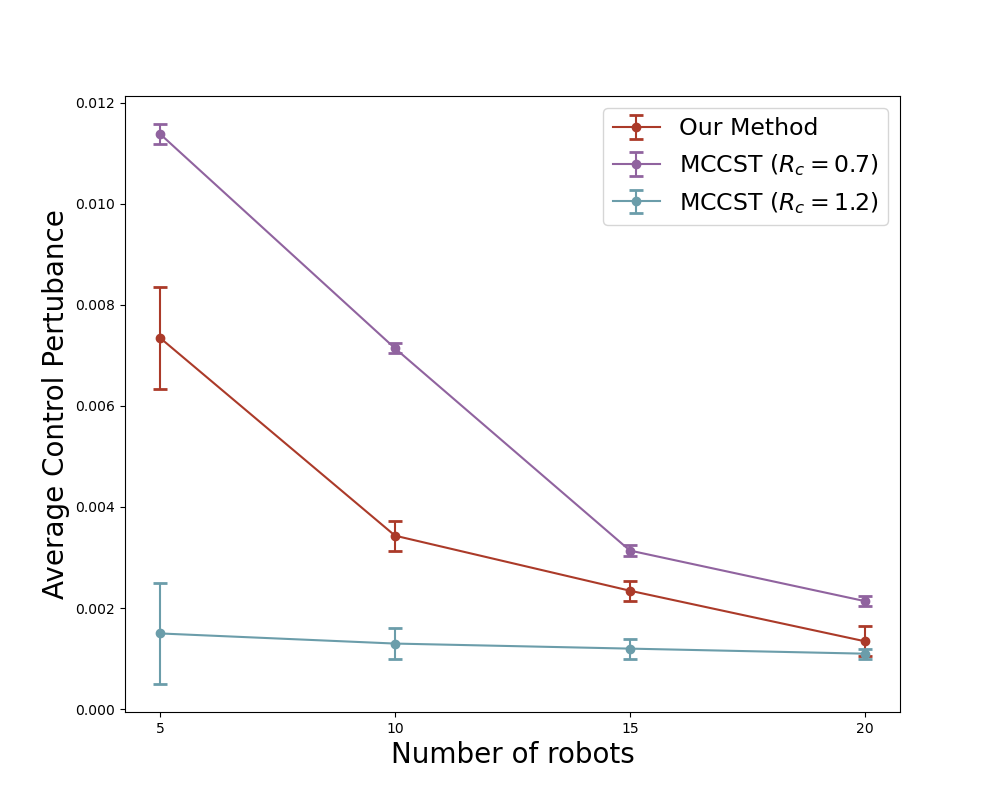}
 \caption{Average control perturbation}
  \label{fig:control_error}
\end{subfigure}
  \caption{Quantitative Results on different sizes of the robot team. Error bars show the standard deviation.}
  \label{fig:qua}
\end{figure*}
\setlength{\textfloatsep}{0pt}

\begin{proposition}\label{proposition:graph}
Assuming $\mathcal{G}(t_0)$ starts as a strongly connected graph, the procedure outlined in Algorithm~\ref{alg:Dynamic} ensures that the resulting communication graph $\mathcal{G}(t)$ will remain strongly connected at all times.
\end{proposition}
\begin{proof}
If the robot team is initially strongly connected, there must exist a particular spanning tree $\bar{\mathcal{T}}_w^\mathrm{c}(t = t_0)\gets$ MST($\mathcal{G}(t = t_0)$) that is strongly connected ($\bar{\mathcal{T}}_w^\mathrm{c}(t = t_0)\subseteq \mathcal{G}(t=t_0)$). Following Algorithm~\ref{alg:Dynamic}, the robots are constrained by the Data-driven Connectivity Barrier Certificates $\mathcal{S}_\mathbf{u}(\mathbf{x},\bar{\mathcal{T}}_w^\mathrm{c})$ with the current MST $ \bar{\mathcal{T}}_w^\mathrm{c}(t= t_0)$ within $t\in[t_0,t_0+\tau]$. With Lemma~\ref{lem:date_Driven_connectivity}, if any Lipschitz continuous controller $\mathbf{u}(\mathbf{x}(t)) \in \mathcal{S}_\mathbf{u}(\mathbf{x},\bar{\mathcal{T}}_w^\mathrm{c}(t=t_0))$ for all $t\in [t_0,t_0 +\tau]$, then it guarantees $\mathbf{x}(t) \in \mathcal{H}^\mathrm{c}(\mathcal{G}^\mathrm{c}(t))$ within $t\in [t_0,t_0+\tau]$. At the next time-step $t_1 = t_0 + \tau$, it is guaranteed that $\bar{\mathcal{T}}_w^\mathrm{c}(t=t_0) \subseteq \mathcal{G}(t=t_1)$. Thus, we can conclude that the step-wise $\bar{\mathcal{T}}_w^\mathrm{los'}(t)$ is always strongly connected resulting in the supergraph $\mathcal{G}(t)$ is always strongly connected. 
\end{proof}

\section{Results}
\subsection{CoppeliaSim Simulation}
To characterize the spatial variance of the wireless signal strength, we built environment as shown in Fig.~\ref{fig1:real_world} for constructing the ground truth data. The detailed data collection process and the experiment video are provided in \url{https://github.com/wenhaol/DCM-RSSI}.
To validate our algorithm, we conducted simulations using CoppeliaSim platform with 5 Khepera IV robots under differential drive dynamics, where controllers from our Data-driven Connectivity Maintenance method are applied using kinematics mapping from \cite{farias2017khepera}. To use the real-world RSSI data, we build the same environment in the CoppeliaSim platform as shown in Fig.~\ref{fig2:subfiga}. We project the robot position to the nearest points in the dataset, to access the corresponding RSSI value. The kernel hyper-parameters in our experiments are set as $\sigma_f = 1$ and $l=0.5$. Besides, we set the $\epsilon = -25$dB and $\psi = -30$dB as the threshold for our considered communication model as mentioned in Section~\ref{sec:communicaiton_model}. Robots with different colors execute biased rendezvous behaviors toward corresponding colored task places. Fig. \ref{fig2:subfiga}-~\ref{fig2:subfigc} demonstrate that our proposed method consistently ensures that $\mathcal{G}$ is strongly connected over time. We also implemented MCCST \cite{luo2020behavior} method with different parameters as shown in Fig.~\ref{fig3:subfigure2}-\ref{fig3:subfigure3}. In Fig.~\ref{fig3:subfigure2} (MCCST with $R_\mathrm{c} = 0.7$m). It is observed that robots can be strongly connected with each other and ensure required safety, however due to predefined limited communication range, robots motion is constrained by the distance-based connectivity control constraints and cannot perform their original tasks well. In Fig.~\ref{fig2:subfigc} (MCCST \cite{luo2020behavior} with $R_\mathrm{c} = 1.2$m), robots cannot be strongly connected with each other due to predefined long communication distance. 
In Fig.~\ref{fig:prediction}, we show the learned $h^\mathrm{gp}_{i,j}(\mathbf{x})$ with its corresponding 0-superlevel set. 
It demonstrates that, with an increasing amount of RSSI data collected from all robots, the learned $h^\mathrm{gp}_{i,j}(\mathbf{x})$ with its corresponding 0-superlevel set is closer to the actual measurement. 
In summary, our method could use online RSSI data to adaptively “change the communication range” to ensure the $\mathcal{G}$ is strongly connected while achieving the best task performance. In Fig.~\ref{fig3:safety}, it further verifies that all three methods could ensure the required safety. Fig.~\ref{fig3:connectivity} shows that MCCST with $R_\mathrm{c} = 1.2$m fails to maintain a strong connection. Besides, in Fig.~\ref{fig3:control} the MCCST with $R_\mathrm{c} = 0.7$m introduces the most average control perturbation, making it hard for robots to perform their designated tasks.
\subsection{Quantitative Results}
To assess the computational efficiency and scalability, we conducted experiments with up to 20 robots, running 10 trials for each batch with varying robot counts. Fig.~\ref{fig:safe_error} demonstrates that safety is guaranteed for all methods. Fig.~\ref{fig:connectivity_error} shows that our method and MCCST with $R_\mathrm{c} = 0.7$m, 
can maintain the required connectivity, while the MCCST with $R_\mathrm{c} = 1.2$m fails ($\lambda_2 = 0$) as this fixed communication range is too far for some transmitter and receiver pairs. Fig.~\ref{fig:control_error} indicates that our method achieves less average control perturbation than $R_\mathrm{c}= 0.7$m.  Hence, our algorithm ensures flexible connectivity maintenance with the best task performance among the robots.
\section{Conclusion}
In this paper, we introduced a novel Data-driven Connectivity Maintenance (DCM) framework that maintains connectivity for multi-robot systems in a realistic environment. This framework encodes the novel Data-driven Connectivity Barrier Certificates, which, by combining Gaussian Processes (GP) and control barrier function (CBF), delineate an admissible control space that ensures a strong connection among the robot team. Simulation results demonstrate the effectiveness of our method. Future directions include decentralizing the DCM algorithm for distributed multi-robot systems.
\bibliographystyle{IEEEtran} 
\bibliography{IROS/ArXiv}

\end{document}